\documentclass[10pt,twocolumn,letterpaper]{article}

\usepackage[pagenumbers]{cvpr} 

\makeatletter
\@namedef{ver@everyshi.sty}{}
\makeatother

\usepackage{times}
\usepackage{epsfig}
\usepackage{graphicx}
\usepackage{amsmath}
\usepackage{amssymb}
\usepackage{booktabs}
\usepackage{amsthm}
\usepackage{multirow}

\usepackage[accsupp]{axessibility}              

\usepackage{todonotes}
\usepackage{calc}

\newtheorem{theorem}{Theorem}
\DeclareMathOperator*{\amin}{min}
\DeclareMathOperator*{\amax}{\,max}
\usepackage{tikz}
\usepackage{pgfplots} 
\pgfplotsset{compat=1.17}
\usepgfplotslibrary{groupplots}

\usepackage[pagebackref,breaklinks,colorlinks]{hyperref}

\usepackage[capitalize]{cleveref}
\crefname{section}{Sec.}{Secs.}
\Crefname{section}{Section}{Sections}
\Crefname{table}{Table}{Tables}
\crefname{table}{Tab.}{Tabs.}

\def\cvprPaperID{40} 
\def\confName{CVPR}
\def\confYear{2023}

\begin{document}
\title{Look ATME: The Discriminator Mean Entropy Needs Attention}

\author{~ \hfill Edgardo Solano-Carrillo \hfill Angel Bueno Rodriguez \hfill Borja Carrillo-Perez \hfill ~\\ ~ \hfill Yannik Steiniger \hfill Jannis Stoppe \hfill ~\\
German Aerospace Center (DLR), Institute for the Protection of Maritime Infrastructures \\
\texttt{\footnotesize\{Edgardo.SolanoCarrillo, Angel.Bueno, Borja.CarrilloPerez, Yannik.Steiniger, Jannis.Stoppe\}@dlr.de}
}

\makeatletter
\let\@oldmaketitle\@maketitle
\renewcommand%
    {\@maketitle}%
    {%
        \@oldmaketitle
        \vspace{-1em}
        \centering
	\includegraphics[width=\linewidth]{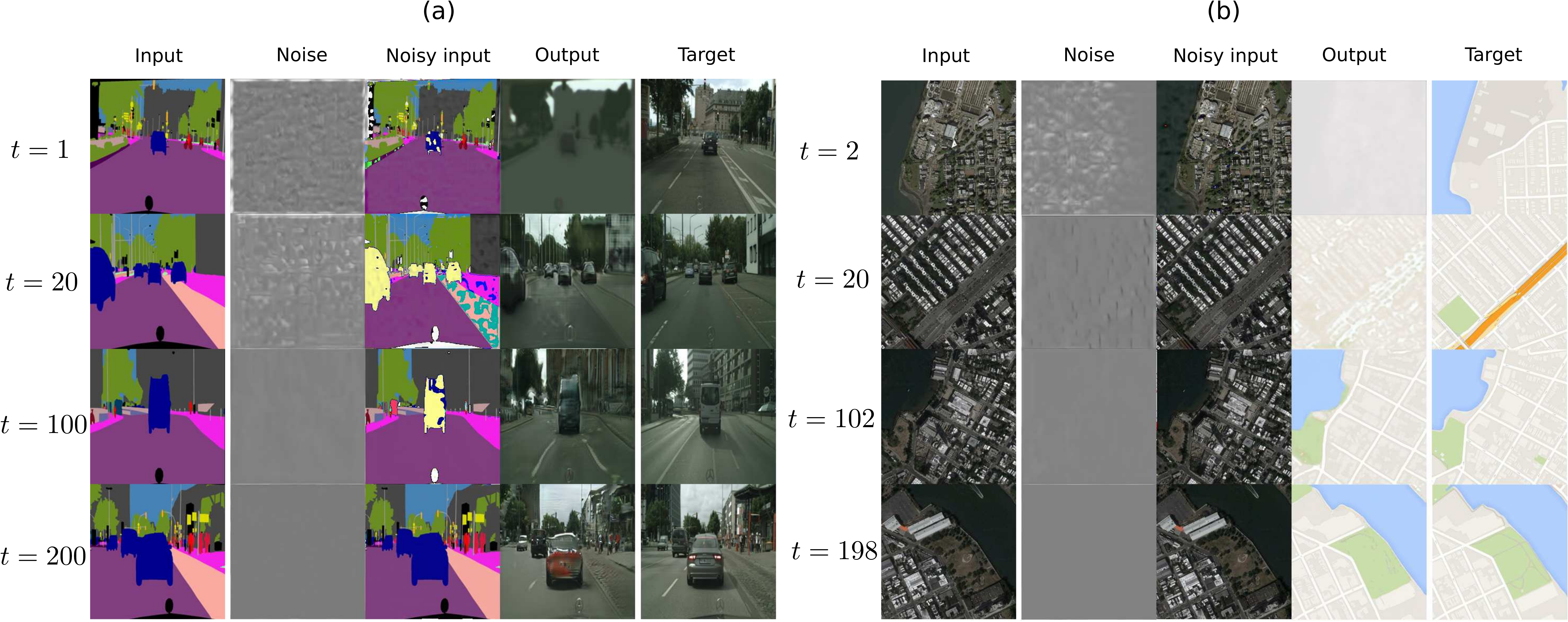}  
        \captionof{figure}%
            {ATME is a GAN where, for each iteration at epoch $t$, the input images for the generator are corrupted with a noisy representation of the discriminator's decision map at epoch $t-1$. The adversarial game enforces this noise to be removed, sending the proper signals for equilibration by encouraging the discriminator to converge towards its maximum entropy state. Convergence is often slower in the cases where the noise significantly affects the semantic content of the input (a) and is faster in the opposite cases (b). %
            }\label{fig:denoising}
            \vspace{0.3cm}
     }
\makeatother

\maketitle

\begin{abstract}
  Generative adversarial networks (GANs) are successfully used for image synthesis but are known to face instability during training. In contrast, probabilistic diffusion models (DMs) are stable and generate high-quality images, at the cost of an expensive sampling procedure. In this paper, we introduce a simple method to allow GANs to stably converge to their theoretical optimum, while bringing in the denoising machinery from DMs. These models are combined into a simpler model (ATME) that only requires a forward pass during inference, making predictions cheaper and more accurate than DMs and popular GANs. ATME breaks an information asymmetry existing in most GAN models in which the discriminator has spatial knowledge of where the generator is failing. To restore the information symmetry, the generator is endowed with knowledge of the entropic state of the discriminator, which is leveraged to allow the adversarial game to converge towards equilibrium. We demonstrate the power of our method in several image-to-image translation tasks, showing superior performance than state-of-the-art methods at a lesser cost. Code is available at \href{https://github.com/DLR-MI/atme}{https://github.com/DLR-MI/atme}.
\end{abstract}


\section{Introduction}
Recent advances in deep learning have led to remarkable progress in the field of image synthesis. Among the most exciting applications, image-to-image translation (where an image in domain A is transformed into a different domain B while preserving the original semantic content) has played a prominent role~\cite{review_img2img_2022}. This task is often addressed using GANs~\cite{gan_goodfellow_nips_2014} or, more recently, with DMs~\cite{ho_diffusion}. Although DMs have been shown to produce high-quality images with unprecedented success, it does so after sequential sampling over multiple time steps. On the other hand, GANs require only a single forward pass for prediction, but suffer from training instabilities that hinder performance. 

In this paper, we propose a novel model for image-to-image translation that harnesses the high-quality image generation power of DMs while eliminating their time-sampling limitation using a GAN. Our approach recognizes that the training instabilities in the latter are rooted in a phenomenon similar to the violation of the second law of thermodynamics by Maxwell's demon~\cite{RevModPhys.81.1, mdemon2017}, and suggests a simple solution to avoid this. 

In order to achieve stable training, we build a GAN whose generator receives images corrupted by a noisy representation of the discriminator's decision map --- as it traverses the training epochs, but not across an independent time-axis as in DMs. The generator then learns to denoise its input to produce the output image, enforcing the discriminator's convergence to its maximum entropy state, as shown in practice by the approach of the GAN (on average) to its theoretical optimum corresponding to the Nash equilibrium ~\cite{pmlr-v119-farnia20a}.   

By learning to diffusively attend to the discriminator mean entropy, our model (ATME) helps to improve training stability by breaking the information asymmetry between the generator and discriminator, leading to better performance in image-to-image translation tasks.

The main contributions of this work are therefore:
\begin{itemize}
 \setlength{\itemsep}{0.25em}
 \item A novel model that fuses the sampling strengths of GANs with the core denoising ideas of DMs into a single efficient model for image-to-image translation.
 \item A practical and simple measure of convergence of GAN models, consistent with the original theoretical description of their optimality.
\end{itemize}

Our approach builds on recent advancements in the field, particularly from diffusion models. These have achieved state-of-the-art performance in image generation~\cite{lehtinen2021diffusion}. Nevertheless, they require thousands of model evaluations to generate high-quality samples~\cite{salimans2022progressive,super_resolution, chen2020wavegrad}. Bridging the gap with GANs is therefore an important step towards enabling high-quality and efficient image-to-image translation for a range of practical applications.

\section{Related work}

\noindent\textbf{GANs for image-to-image translation}. GANs have been for a long time the de facto method for generation of synthetic images~\cite{review_img2img_2022}. pix2pix~\cite{isola2017image} was the first unified framework for supervised image-to-image translation using conditional GANs. It serves as a foundational model on top of which other solutions have been built, such as adding cycle consistency to a couple of GANs, \ie CycleGAN~\cite{zhu2017unpaired}, for unsupervised image-to-image translation. The latter has further inspired other models such as UNIT~\cite{unit_nips_2017}, which leverages a latent representation of the support of the joint distribution of the unpaired images, and several other multi-domain variants~\cite{choi_stargan, Huang_2018_ECCV, Liu_2019_ICCV}. Of special importance for this work is the use of attention mechanisms in GANs. In particular, FAL~\cite{Huh_2019_CVPR} improves image synthesis with a generator that repeatedly receives feedback---in several forward passes---from the discriminator. SPA-GAN~\cite{spa_gan_2021}, computes attention in the discriminator to help the generator focus on the most discriminative regions between source and target domains. Most recently, ASGIT~\cite{Lin_2021_WACV} also enforces spatial guidance by adding attention in the discriminator, surpassing previous methods for supervised and unsupervised image-to-image translation.

Our approach builds on pix2pix, recognizing the information advantages of its patch discriminator, which is counterbalanced by adding attention to the generator. Since the main focus in this work is the effect on convergence, we study this in a supervised setting.

\noindent\textbf{Convergence during GAN training}. Several proposals have been made to address the stability issues posed by training GANs, which include vanishing or exploiting gradients and mode collapse. These typically manifest as an ill-behaved Jacobian of the  gradient vector field of the associated GAN objectives~\cite{mescheder_nips2017}. To address this, SNGAN~\cite{miyato2018spectral} proposes a weight normalization technique called spectral normalization to stabilize the training of the discriminator. On the other hand, WGAN~\cite{wgan_icml_2017} introduces the Wasserstein distance between real and fake distributions as an objective to optimize, alleviating the mode collapse problem of vanilla GANs~\cite{gan_goodfellow_nips_2014}, which optimize the Jensen-Shannon divergence.  WGAN-GP~\cite{wgan_gp_nips_2017} improves training in practice by adding a gradient penalty to enforce the required discriminator 1-Lipschitz constraint. Alternatively, LSGAN~\cite{mao2017least} optimizes the Pearson $\chi^{2}$ divergence between the real and fake distributions. Viewing the convergence in GAN training as a matter of finding the right divergence to minimize at each step is misleading though ~\cite{fedus2017many}; more beneficial convergence characteristics are found in practice by adding instance noise or gradient penalties~\cite{mescheder2018training}.

Architectures also play a role in the stability of GAN training. Energy-based GANs view the discriminator as an energy function taking on lower values for regions near the data manifold. By using the reconstruction error of an autoencoder as an energy function, EBGANs~\cite{zhao2017energybased} exhibit more stable behavior than vanilla GANs. After approximating the Wasserstein distance using autoencoders,  BEGAN~\cite{berthelot2017began} intends to balance the generator and discriminator during training. RGANs~\cite{jolicoeur_iclr_2019} make the discriminator relativistic (i.e. discriminating whether real data is more realistic than fake data) making training more stable. 

Rather than improving network architectures, or changing the objectives functions for training, or regularizing gradients/weights; our work focuses on vanilla GANs with standard networks, stabilizing training by symmetrizing the information exchange between the GAN adversaries.

\vspace{0.1cm}\ \\
\noindent\textbf{Diffusion probabilistic models}. Diffusion models are generative models that iteratively transform a random noise distribution into a target data distribution by learning a reverse denoising process~\cite{ho_diffusion,dicksteien_thermo,song_datadistri}. They have arisen as the current state of the art in the field of synthetic data generation~\cite{survey_diffusion}, surpassing GANs~\cite{lehtinen2021diffusion} in the quality of image synthesis --- after denoising either directly in the image space~\cite{kingma2021variational} or in a latent representation, such as latent diffusion models (LDMs)~\cite{latent_diffusion}. In the context of high-quality image generation leveraging intance noise injection, combining ideas from diffusion models with GANs has gained current research traction~\cite{denoising-diffusion-gans,diffusion-GAN}. However, the cost of the sampling procedure in diffusion-based models still remains an issue, which may be mitigated by modeling the denoising distribution as a complex multimodal distribution instead of a Gaussian~\cite{denoising-diffusion-gans}, or by making the number of timesteps dependent on the data and the generator~\cite{diffusion-GAN}.


Our approach for injecting instance noise is not based on an independent and expensive diffusion process. It is rather the iterative visit of the data distribution through the training epochs that occurs diffusively, after corrupting the generator inputs with a representation of the disorder state of the discriminator outputs.

\section{Background}

\subsection{Conditional GANs}
The generator $G$ in these models learn a mapping from the image $x$ and noise vector $z$ to the image $y$. Its output is discriminated by a model $D$, judging whether the image is real or fake. The objective is
\begin{equation}\label{eq:GAN}
\begin{split}
 \tilde{\mathcal{L}}_{GAN}(G,D)=&\ \mathbb{E}_{x,y}[\log D(x,y)]+\\
 &\ \mathbb{E}_{x,z}[\log(1- D(x,G(x,z)))],
 \end{split}
\end{equation}
where $G$ is trained to minimize this objective and $D$ is trained to maximize it, known as the min-max game.

With the introduction of the patchGAN discriminator in pix2pix~\cite{isola2017image}, the discriminator outputs a tensor (default size of $30\times30$), with each entry $D_i$ classifying a patch (receptive field size of $70\times70$) in the input image. With $N$ being the number of patches, the objective becomes
\begin{equation}\label{eq:patchGAN}
\begin{split}
 \mathcal{L}_{GAN}(G,D)=\dfrac{1}{N}\sum_{i=1}^N\;\tilde{\mathcal{L}}_{GAN}(G,D_{i}).
 \end{split}
\end{equation}
The motivation of discriminating by patches is enforcing the generator to produce correct high-frequency patterns, while the low frequencies are captured by a $L1$ penalty
\begin{equation}
 \mathcal{L}_{L1}(G) = \mathbb{E}_{x,y,z}[\|G(x,z)-y\|_1].
\end{equation}

\begin{figure*}
 \centering
 \includegraphics[width=0.95\linewidth]{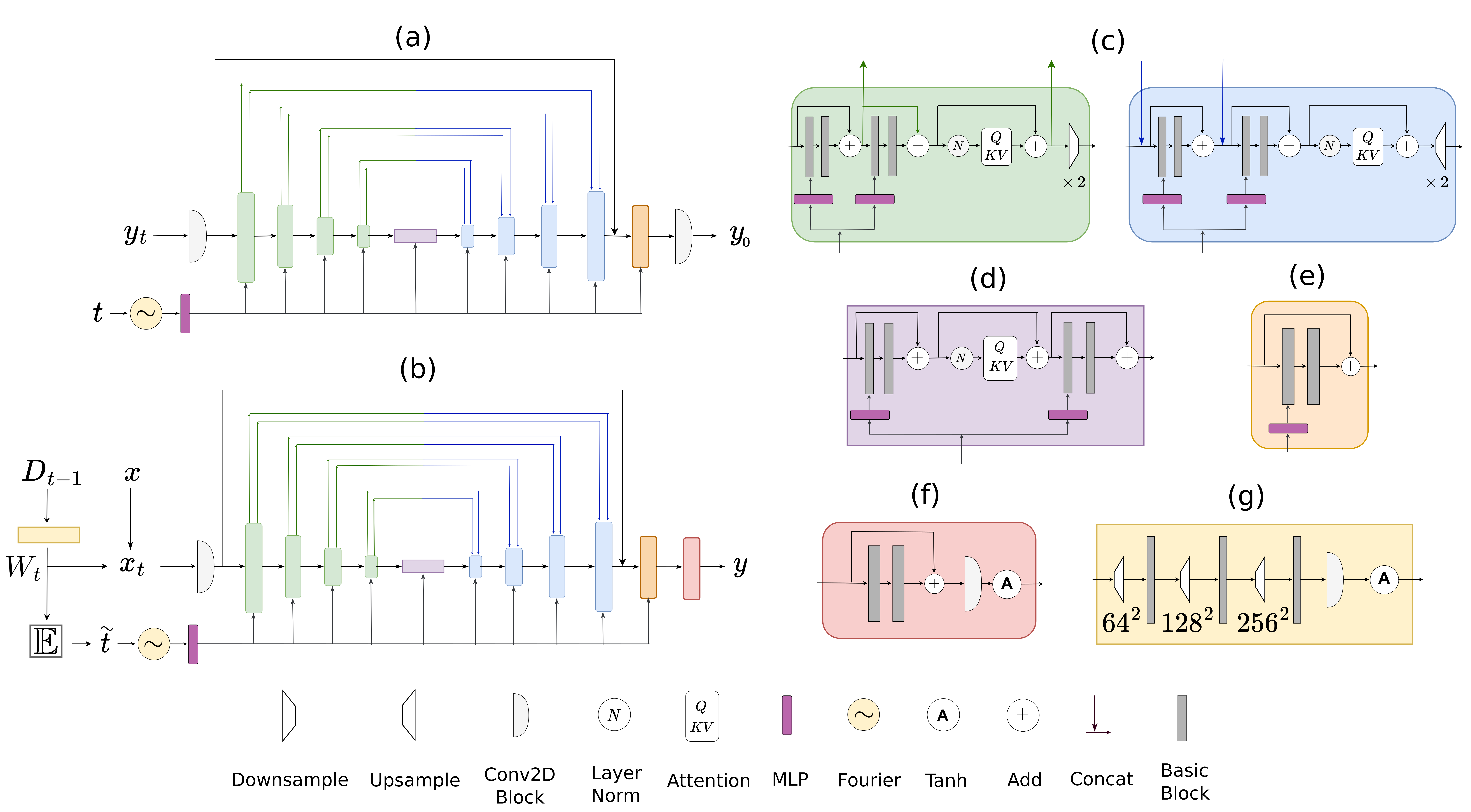}
 
 \caption{Building blocks (c)-(e) for the denoising UNet (a) used in diffusion models~\cite{ho_diffusion}. Our model (b) adds new building blocks for its generator, consisting of (g) for listening to the discriminator entropic state in order to corrupt the input image with noise, and a modified head (f) to remove spurious high-frequency patterns. All building blocks are based on a basic block that convolves the input, normalizes the resulting channels, optionally applies an affine transformation (using the Fourier features), and finally applies SiLU activation.}
 \label{fig:backbone}
\end{figure*}
The final objective is then
\begin{equation}\label{eq:gan_obj}
 \arg\amin_{G}\amax_{D}\;\;\mathcal{L}_{GAN}(G,D) + \lambda\,\mathcal{L}_{L1}(G),
\end{equation}
with $\lambda$ typically chosen as $\lambda=100$.

\subsection{Diffusion models}
Diffusion models are generative models designed to learn a data distribution $p(y_0)$ by sequentially denosing a normally-distributed variable $y_t\sim \alpha_t y_0 + \sigma_t \varepsilon$, by using a model $y_{\theta}=y_{\theta}(y_t, t)$ with the objective
\begin{equation}\label{eq:Ldm}
  \mathcal{L}_{DM}\propto\mathbb{E}_{\varepsilon,t}\|y_0-y_{\theta}(y_t, t)\|^2.
\end{equation}

Here the sequences $(\alpha_t)_{t=1}^{T}$ and $(\sigma_t)_{t=1}^{T}$ are chosen following a schedule that makes the signal-to-noise ratio $\textrm{SNR}(t)=\alpha_t^2/\sigma_t^2$ small enough at $t=T$ (typically $\textrm{SNR}(T)\,\|y_0\|^2\sim 10^{-5}$) and $\varepsilon\sim\mathcal{N}(0,1)$. In practice, only one schedule for the variable $\beta_t$ in $\bar{\alpha}_t=\prod_{s=1}^t(1-\beta_s)$ is chosen, with $\textrm{SNR}(t)=\bar{\alpha}_t/(1-\bar{\alpha}_t)$.

\vspace{1.5em}
\section{Attending the mean entropy (ATME)}
Introducing a discrimination by patches allows the discriminator to have a notion of \emph{where} the generator is failing. This makes the min-max game asymmetric in favor of the discriminator, since the generator has no direct spatial clue of where the discriminator is failing. Without further intervention, this forbids a proper equilibration of the game, resulting in a lack of convergence. Our task is to find the piece of information about the discriminator that the generator should know in order to recover the symmetry.

The situation is similar to the information asymmetry introduced in statistical physics by Maxwell's demon. That is, when two ideal gases at different temperatures are placed in separate containers communicated by a switchable hole, equilibration (corresponding to the maximum entropy state) is achieved when the hole is opened --- more fast-moving particles moving from the hot container to the cold one than backwards. But if an entity (demon) is introduced, which opens the hole to allow the backward motion and close it to block the forward, the cold container will be colder and the hot container hotter, and equilibration never takes place.

In the GAN game, the information gain introduced by the patch discriminator is analog to the information gain of Maxwell's  demon due to its knowledge of the velocity of the particles in both containers. We propose to incentivate a proper equilibration by letting the generator enforce the corresponding maximum entropy state --- seeing the Nash equilibrium~\cite{pmlr-v119-farnia20a} as a thermal equilibrium. The following fact (proved in the appendix) hints us on how to achieve this:  

\begin{theorem}
Let $Y_{i}$ be a binary random variables taking on the value $y_{i}=1$ with probability $D_{i}$. If they are statistical independent, the joint distribution $P(Y_1,\cdots,Y_N)$ has maximum entropy if and only if $D_{i}=\tfrac{1}{2}$ for all $i$. In this state, the objective in \cref{eq:patchGAN} reaches the value $-\log(4)$ for an optimal discriminator and generator.
\label{th:max_ent}
\end{theorem}

We propose to endow the generator with a notion of the state of  ``disorder'' of the discriminator decisions, this being a surrogate to its entropy. Denoting by $D_t$ the output tensor of the discriminator at training epoch $t$ (having entries $D_{i;t}$), we introduce the learnable mapping $W_t=W(D_{t-1})$ having a range in the space of the input images of the generator. This should have the following properties:
\begin{itemize}
 \item As $D_t$ tends to the maximum entropy state, $W(D_t)$ tends to a constant tensor, and viceversa. That is, as $D_{i;t}\rightarrow\tfrac{1}{2}$ for all patches $i$, $W_{r;t}\rightarrow w$ for all pixels $r$. This is our statement of the preservation of the state of disorder under the action of $W$.
 \item The differences $W(D_t)-W(D_{t-1})$ are uncorrelated in time and approximately Gaussian.
\end{itemize}

The second property is a weaker one, only ensuring that the input images for the generator, which we take as
\begin{equation}\label{eq:xt}
 x_t=x_0 + x_0 \,W(D_{t-1}),
\end{equation}
initially follow a Brownian motion, diffusing through the epochs during training. This allows us to borrow the intuition from the diffusion models. That is, we corrupt the input image $x_0=x$ with ``noise'' arising from $W(D_{t-1})$ and train the generator to get rid of this noise in order to capture the correct mapping $x\rightarrow y$ (as shown in \cref{fig:denoising}). As a side effect, removing this noise sends the signal to the discriminator to seek the maximum entropy state, by the main property of $W(D_t)$.   

\begin{figure*}
	\begin{tikzpicture}
		\begin{groupplot}[
			group style={group size= 2 by 4},
			height=3cm,
			width=.5\linewidth,
			no markers,
			every axis plot/.append style={thick},
			every axis/.append style={font=\footnotesize},
			ymax=2.6,
			xmin=-5,
			xmax=205,
			ytick={0,1,ln(4),2},
			yticklabels={0,1,$log(4)$,2},
			ylabel near ticks,
			xlabel near ticks,
			legend columns=-1
			]
			
			\newcommand{\cycleganlinestyle}{dashed}
			\newcommand{\cyclegancolor}{cyan}
			\newcommand{\asgitlinestyle}{densely dashed}
			\newcommand{\asgitcolor}{violet}
			\newcommand{\pixtopixlinestyle}{densely dotted}
			\newcommand{\pixtopixcolor}{orange}
			\newcommand{\atmelinestyle}{solid}
			\newcommand{\atmecolor}{teal}
			\newcommand{\unitlinestyle}{dotted}
			\newcommand{\unitlinecolor}{blue}
			\newcommand{\loglinecolor}{gray}
			
			\nextgroupplot[title=A$\rightarrow$B, ylabel=Facades,legend to name=lgnd]
			\addplot+ [\loglinecolor,forget plot] table [x expr=\coordindex+1, y expr=ln(4)]{make_figs/loss_D_facades_atme_BA_smooth.csv};
			\addplot+ [cyan,dashed] table [x expr=\coordindex+1, y=loss_D, col sep=comma] {make_figs/loss_D_facades_cyclegan_AB_smooth.csv};
			\addlegendentry{CycleGAN};
			\addplot+ [\asgitcolor,\asgitlinestyle] table [x expr=\coordindex+1, y=loss_D, col sep=comma] {make_figs/loss_D_facades_ASGIT_AB_smooth.csv};
			\addlegendentry{ASGIT};
			\addplot+ [\pixtopixcolor,\pixtopixlinestyle] table [x expr=\coordindex+1, y=loss_D, col sep=comma] {make_figs/loss_D_facades_pix2pix_AB_smooth.csv};
			\addlegendentry{pix2pix};
			\addplot+ [\unitlinecolor,\unitlinestyle] table [x expr=\coordindex+1, y=loss_D, col sep=comma] {make_figs/loss_D_facades_unit_AB_smooth.csv};
			\addlegendentry{UNIT};
			\addplot+ [\atmecolor,\atmelinestyle] table [x expr=\coordindex+1,  y=loss_D, col sep=comma] {make_figs/loss_D_facades_atme_AB_smooth.csv};
			\addlegendentry{ATME};
			\coordinate (top) at (rel axis cs:0,1);
			
			\nextgroupplot[title=B$\rightarrow$A]
			\addplot+ [\loglinecolor] table [x expr=\coordindex+1, y expr=ln(4)]{make_figs/loss_D_facades_atme_BA_smooth.csv};
			\addplot+ [\cyclegancolor,\cycleganlinestyle] table [x expr=\coordindex+1, y=loss_D, col sep=comma] {make_figs/loss_D_facades_cyclegan_BA_smooth.csv};
			\addplot+ [\asgitcolor,\asgitlinestyle] table [x expr=\coordindex+1, y=loss_D, col sep=comma] {make_figs/loss_D_facades_ASGIT_BA_smooth.csv};
			\addplot+ [\pixtopixcolor,\pixtopixlinestyle] table [x expr=\coordindex+1, y=loss_D, col sep=comma] {make_figs/loss_D_facades_pix2pix_BA_smooth.csv};
			\addplot+ [\unitlinecolor,\unitlinestyle] table [x expr=\coordindex+1, y=loss_D, col sep=comma] {make_figs/loss_D_facades_unit_BA_smooth.csv};
			\addplot+ [\atmecolor,\atmelinestyle] table [x expr=\coordindex+1,  y=loss_D, col sep=comma] {make_figs/loss_D_facades_atme_BA_smooth.csv};
			
			\nextgroupplot[ylabel=Maps]
			\addplot+ [\loglinecolor] table [x expr=\coordindex+1, y expr=ln(4)]{make_figs/loss_D_maps_atme_BA_smooth.csv};
			\addplot+ [\cyclegancolor,\cycleganlinestyle] table [x expr=\coordindex+1, y=loss_D, col sep=comma] {make_figs/loss_D_maps_cyclegan_AB_smooth.csv};
			\addplot+ [\asgitcolor,\asgitlinestyle] table [x expr=\coordindex+1, y=loss_D, col sep=comma] {make_figs/loss_D_maps_ASGIT_AB_smooth.csv};
			\addplot+ [\pixtopixcolor,\pixtopixlinestyle] table [x expr=\coordindex+1, y=loss_D, col sep=comma] {make_figs/loss_D_maps_pix2pix_AB_smooth.csv};
			\addplot+ [\unitlinecolor,\unitlinestyle] table [x expr=\coordindex+1, y=loss_D, col sep=comma] {make_figs/loss_D_maps_unit_AB_smooth.csv};
			\addplot+ [\atmecolor,\atmelinestyle] table [x expr=\coordindex+1,  y=loss_D, col sep=comma] {make_figs/loss_D_maps_atme_AB_smooth.csv};
			
			\nextgroupplot[]
			\addplot+ [\loglinecolor] table [x expr=\coordindex+1, y expr=ln(4)]{make_figs/loss_D_maps_atme_BA_smooth.csv};
			\addplot+ [\cyclegancolor,\cycleganlinestyle] table [x expr=\coordindex+1, y=loss_D, col sep=comma] {make_figs/loss_D_maps_cyclegan_BA_smooth.csv};
			\addplot+ [\asgitcolor,\asgitlinestyle] table [x expr=\coordindex+1, y=loss_D, col sep=comma] {make_figs/loss_D_maps_ASGIT_BA_smooth.csv};
			\addplot+ [\pixtopixcolor,\pixtopixlinestyle] table [x expr=\coordindex+1, y=loss_D, col sep=comma] {make_figs/loss_D_maps_pix2pix_BA_smooth.csv};
			\addplot+ [\unitlinecolor,\unitlinestyle] table [x expr=\coordindex+1, y=loss_D, col sep=comma] {make_figs/loss_D_maps_unit_BA_smooth.csv};
			\addplot+ [\atmecolor,\atmelinestyle] table [x expr=\coordindex+1,  y=loss_D, col sep=comma] {make_figs/loss_D_maps_atme_BA_smooth.csv};
			
			\nextgroupplot[ylabel=Cityscapes]
			\addplot+ [gray] table [x expr=\coordindex+1, y expr=ln(4)]{make_figs/loss_D_cityscapes_atme_BA_smooth.csv};
			\addplot+ [\cyclegancolor,\cycleganlinestyle] table [x expr=\coordindex+1, y=loss_D, col sep=comma] {make_figs/loss_D_cityscapes_cyclegan_AB_smooth.csv};
			\addplot+ [\asgitcolor,\asgitlinestyle] table [x expr=\coordindex+1, y=loss_D, col sep=comma] {make_figs/loss_D_cityscapes_ASGIT_AB_smooth.csv};
			\addplot+ [\pixtopixcolor,\pixtopixlinestyle] table [x expr=\coordindex+1, y=loss_D, col sep=comma] {make_figs/loss_D_cityscapes_pix2pix_AB_smooth.csv};
			\addplot+ [\unitlinecolor,\unitlinestyle] table [x expr=\coordindex+1, y=loss_D, col sep=comma] {make_figs/loss_D_cityscapes_unit_AB_smooth.csv};
			\addplot+ [\atmecolor,\atmelinestyle] table [x expr=\coordindex+1,  y=loss_D, col sep=comma] {make_figs/loss_D_cityscapes_atme_AB_smooth.csv};
			
			\nextgroupplot[]
			\addplot+ [\loglinecolor] table [x expr=\coordindex+1, y expr=ln(4)]{make_figs/loss_D_cityscapes_atme_BA_smooth.csv};
			\addplot+ [\cyclegancolor,\cycleganlinestyle] table [x expr=\coordindex+1, y=loss_D, col sep=comma] {make_figs/loss_D_cityscapes_cyclegan_BA_smooth.csv};
			\addplot+ [\asgitcolor,\asgitlinestyle] table [x expr=\coordindex+1, y=loss_D, col sep=comma] {make_figs/loss_D_cityscapes_ASGIT_BA_smooth.csv};
			\addplot+ [\pixtopixcolor,\pixtopixlinestyle] table [x expr=\coordindex+1, y=loss_D, col sep=comma] {make_figs/loss_D_cityscapes_pix2pix_BA_smooth.csv};
			\addplot+ [\unitlinecolor,\unitlinestyle] table [x expr=\coordindex+1, y=loss_D, col sep=comma] {make_figs/loss_D_cityscapes_unit_BA_smooth.csv};
			\addplot+ [\atmecolor,\atmelinestyle] table [x expr=\coordindex+1,  y=loss_D, col sep=comma] {make_figs/loss_D_cityscapes_atme_BA_smooth.csv};
			\coordinate (bot) at (rel axis cs:1,0);

			\nextgroupplot[ylabel=Night2day]
			\addplot+ [\loglinecolor] table [x expr=\coordindex+1, y expr=ln(4)]{make_figs/loss_D_night2day_atme_BA_smooth.csv};
			\addplot+ [\cyclegancolor,\cycleganlinestyle] table [x expr=\coordindex+1, y=loss_D, col sep=comma] {make_figs/loss_D_night2day_cyclegan_AB_smooth.csv};
			\addplot+ [\asgitcolor,\asgitlinestyle] table [x expr=\coordindex+1, y=loss_D, col sep=comma] {make_figs/loss_D_night2day_ASGIT_AB_smooth.csv};
			\addplot+ [\pixtopixcolor,\pixtopixlinestyle] table [x expr=\coordindex+1, y=loss_D, col sep=comma] {make_figs/loss_D_night2day_pix2pix_AB_smooth.csv};
			\addplot+ [\unitlinecolor,\unitlinestyle] table [x expr=\coordindex+1, y=loss_D, col sep=comma] {make_figs/loss_D_night2day_unit_AB_smooth.csv};
			\addplot+ [\atmecolor,\atmelinestyle] table [x expr=\coordindex+1,  y=loss_D, col sep=comma] {make_figs/loss_D_night2day_atme_AB_smooth.csv};
			
			\nextgroupplot[]
			\addplot+ [\loglinecolor] table [x expr=\coordindex+1, y expr=ln(4)]{make_figs/loss_D_maps_atme_BA_smooth.csv};
			\addplot+ [\cyclegancolor,\cycleganlinestyle] table [x expr=\coordindex+1, y=loss_D, col sep=comma] {make_figs/loss_D_night2day_cyclegan_BA_smooth.csv};
			\addplot+ [\asgitcolor,\asgitlinestyle] table [x expr=\coordindex+1, y=loss_D, col sep=comma] {make_figs/loss_D_night2day_ASGIT_BA_smooth.csv};
			\addplot+ [\pixtopixcolor,\pixtopixlinestyle] table [x expr=\coordindex+1, y=loss_D, col sep=comma] {make_figs/loss_D_night2day_pix2pix_BA_smooth.csv};
			\addplot+ [\unitlinecolor,\unitlinestyle] table [x expr=\coordindex+1, y=loss_D, col sep=comma] {make_figs/loss_D_night2day_unit_BA_smooth.csv};
			\addplot+ [\atmecolor,\atmelinestyle] table [x expr=\coordindex+1,  y=loss_D, col sep=comma] {make_figs/loss_D_night2day_atme_BA_smooth.csv};
			\coordinate (bot) at (rel axis cs:1,0);
	
		\end{groupplot}
		\path (top)--(bot) coordinate[midway] (group center);
		\node[above,rotate=90] at (group center -| current bounding box.west) {$-\mathcal{L}_{GAN}(G,D)$};
		\node [below] (epoch) at (group center |- current bounding box.south) {epoch};
		\node[below, yshift=-.5em] at(epoch) {\pgfplotslegendfromname{lgnd}};
	\end{tikzpicture}
	\caption{Smoothed $-\mathcal{L}_{GAN}(G,D)$ from \cref{eq:patchGAN} at the end of each training epoch for all GAN models and datasets considered. According to \cref{eq:gan_obj} and \cref{th:max_ent}, this should converge --- in the limit of a large enough model and infinite data~\cite{goodfellow2016nips} --- to the Nash equilibrium, where the value $\log(4)$ is reached. Being the lightest of all, our model ATME shows better converge properties (on average) except in the largest dataset Night2day, where it lacks capacity to accommodate all the data variability. Also, its convergence is slower in the tasks B$\rightarrow$A where it is harder to remove the noise applied to the input images due to this noise significantly altering their semantic content.}
	\label{fig:losses}
\end{figure*}
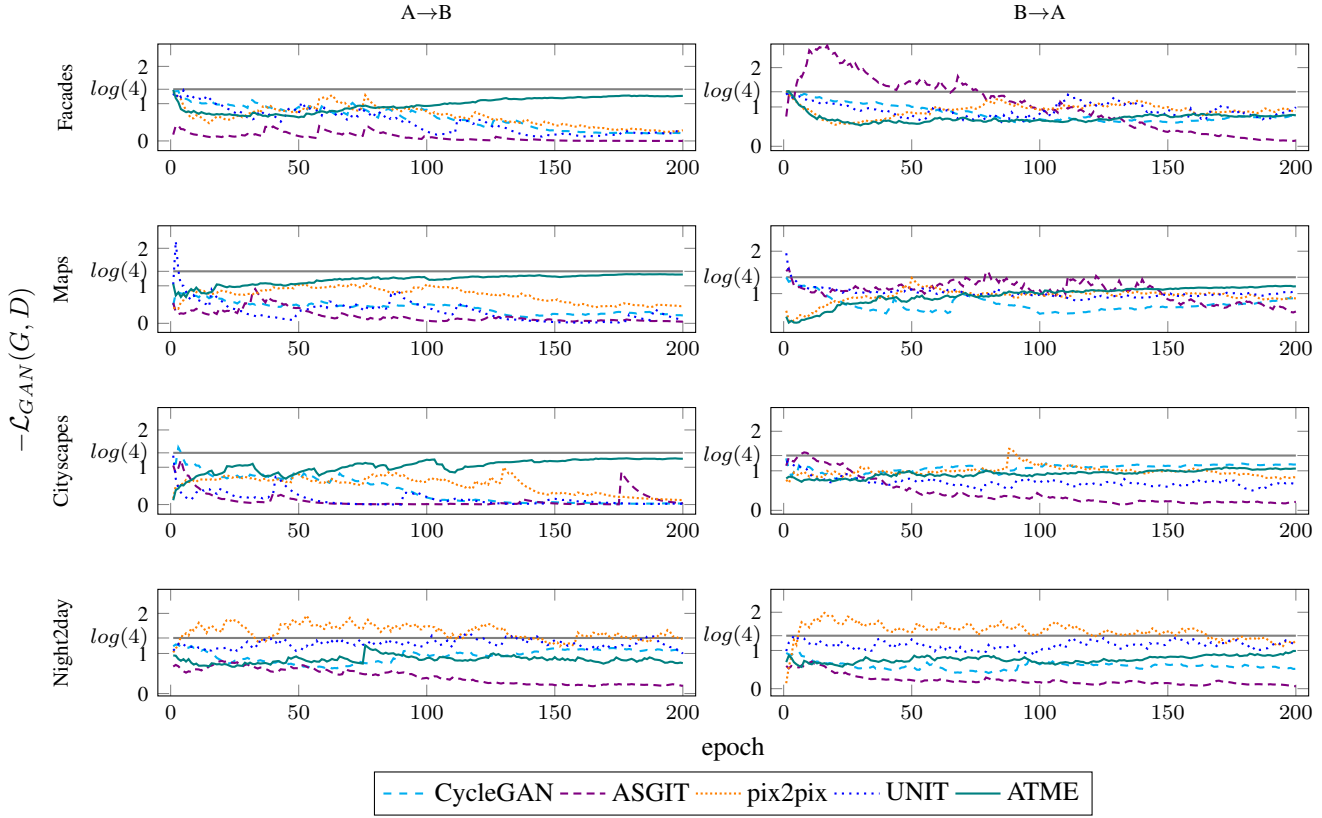

It is important to note that, although the epochs index the time steps $t$ in the experiments, the arrow of time set in the generator has to follow the discriminator's entropic state. This is achieved by estimating the temporal position of the noising events according to
\begin{equation}
 \tilde{t}=\mathbb{E} \,[W(D_{t-1})],
\end{equation}
which is similar to the ordering imposed in the diffusion models by $\textrm{SNR}(T)$ being small and $\textrm{SNR}(0)$ being large.

The loss of ATME at epoch $t$ is then, similar to \cref{eq:gan_obj},
\begin{equation}
\mathcal{L}_{\textrm{ATME}}^t(G,D)=\mathcal{L}_{GAN}^t(G,D)+\lambda\,\mathcal{L}_{L1}^t(G),
\end{equation}
with the superindex indicating that the variables $(x,z)$ are replaced by the combined variable $x_t$ in \cref{eq:xt}, and the generator acquires the functional form (see \cref{eq:Ldm}) that is used in the diffussion models, $G(x_t):=y_{\theta}(x_t, \tilde{t}\,)$.  

At inference, $D_{t-1}$ is sampled element-wise from a normal distribution with mean $\tfrac{1}{2}$ (the maximum-entropy value) and small standard deviation (set to $0.001$ in all experiments).

\subsection{Model architecture}
The architecture of the patch discriminator in ATME follows the implementation of pix2pix~\cite{isola2017image}. The generator is shown in \cref{fig:backbone}(b). It has the UNet structure used in the diffusion models (see \cref{fig:backbone}(a)) with additional blocks that we introduce to listen to the discriminator's entropic state, \cref{fig:backbone}(g), and a modified head (\cref{fig:backbone}(f)) to remove spurious high-frequency patterns.

The UNet is mainly parameterized by an embedding dimension $d$ and attention resolutions $R=(r_1, r_2, \cdots, r_H)$, with $H$ being half the depth of the network (excluding the middle block). The notation means that at the $i$th downsampling layer, the number of feature maps go from $d\,r_{i-1}$ to $d\,r_i$ (with $r_0=1$). The default network for all the experiments has $d=64$ with $R=(1,2,4,8)$, corresponding to a network with four downsampling layers, a middle block, and four upsampling layers, as shown in \cref{fig:backbone}(b).

\section{Experiments}
\subsection{Datasets}
We use four of the standard datasets for supervised image-to-image translation: Facades, Maps, Cityscapes, and Night2day, whose details can be found in~\cite{isola2017image}. For Night2day, we train only on 5000 images. Image-to-image translation is performed in both directions A$\rightarrow$B and B$\rightarrow$A, as defined in \Cref{tab:dataset_AB}.

\begin{table}[t]
\begin{center}
\begin{small}
\begin{tabular}{ccc}
\toprule
\textbf{Dataset} & \textbf{A} & \textbf{B}\\
\midrule
Facades & Photo & Architectural labels  \\
Maps & Aerial photo & Map  \\ 
Cityscapes & Photo & Semantic labels  \\ 
Night2day & Night photo & Day photo\\
\bottomrule
\end{tabular}
\end{small}
\end{center}
\vspace{-1em}
\caption{Datasets used in this work. The corresponding images are paired, \ie the first half of the width of the image is called A and the second half is called B.} 
\label{tab:dataset_AB}
\end{table}

\setlength{\tabcolsep}{5pt}
\begin{table*}[t]
\begin{center}
\begin{small}

\begin{tabular}{cccccccccc}
\toprule
\multirow{2}{*}{\textbf{Model}} &
\multicolumn{1}{c}{\textbf{\# Params}} &
\multicolumn{2}{c}{\textbf{Facades}} &
\multicolumn{2}{c}{\textbf{Maps}} &
\multicolumn{2}{c}{\textbf{Cityscapes}} &
\multicolumn{2}{c}{\textbf{Night2day}} \\
&[M] & A$\rightarrow$B & B$\rightarrow$A & A$\rightarrow$B & B$\rightarrow$A & A$\rightarrow$B & B$\rightarrow$A & A$\rightarrow$B & B$\rightarrow$A\\
\midrule
pix2pix & 57 & 31.3 $\pm$ 2.3 & 11.0 $\pm$ 0.7 & 25.7 $\pm$ 2.0 & 19.0 $\pm$ 1.8 & 16.0 $\pm$ 0.8 & 7.8 $\pm$ 1.0 & 19.2 $\pm$ 1.6 &  11.6 $\pm$ 1.1\\
CycleGAN & 114 & 28.1 $\pm$ 2.1 & 18.2 $\pm$ 1.0 & 60.5 $\pm$ 1.3 & 10.8 $\pm$ 1.4 & 45.0 $\pm$ 1.0 & 16.9 $\pm$ 1.3 & \textbf{12.6 $\pm$ 2.0} &  \textbf{9.0 $\pm$ 0.9}\\
UNIT & 116 & 47.9 $\pm$ 2.4 & 18.0 $\pm$ 0.9 & 30.1 $\pm$ 0.9 & 9.3 $\pm$ 1.1 & 16.6 $\pm$ 1.1 & 12.8 $\pm$ 1.0 & 15.6 $\pm$ 2.1 & 19.7 $\pm$ 1.5 \\
ASGIT & 57& 22.6 $\pm$ 1.6 & \textbf{4.9 $\pm$ 0.8} & 9.2 $\pm$ 1.2 & 7.7 $\pm$ 1.2 & 16.0 $\pm$ 1.3 & \textbf{4.2 $\pm$ 0.5} & 17.5 $\pm$ 2.2 & 11.1 $\pm$ 1.3\\

LDM & 270& 30.9 $\pm$ 2.4 & 23.0 $\pm$ 1.0 & 7.9 $\pm$ 1.0 & 10.3 $\pm$ 1.3 & \textbf{5.6 $\pm$ 0.6} & 5.6 $\pm$ 0.5 & 19.6 $\pm$ 2.2 & 11.9 $\pm$ 1.3\\

ATME & \textbf{39}& \textbf{18.4 $\pm$ 1.8} & 9.4 $\pm$ 0.7 & \textbf{2.8 $\pm$ 0.6} & \textbf{2.8 $\pm$ 0.7} & 6.5 $\pm$ 1.0 & 5.7 $\pm$ 0.8 & 19.7 $\pm$ 2.1 & 18.3 $\pm$ 1.4\\

\bottomrule
\end{tabular}
\end{small}
\end{center}
\caption{KID scores (scaled by 100) for the methods evaluated in the datasets shown (lower is better). The best result per column is shown in bold. Our model ATME, shows superior performance, defined as the number of times it has the best KID per task.}
\label{tab:kid_all}
\end{table*}

\begin{table}[t]
\begin{center}
\begin{small}
\begin{tabular}{cccc}
\toprule
\textbf{Model} & \textbf{Per-pixel acc.} & \textbf{Per-class acc.} & \textbf{Class IoU}\\
\midrule
pix2pix & 0.63 & 0.18 & 0.13 \\
CycleGAN & 0.49 & 0.13 & 0.09 \\ 
UNIT & 0.48 & 0.12 & 0.09 \\
ASGIT & 0.54 & 0.17 & 0.11 \\ 
LDM & 0.57 & 0.17 & 0.11 \\ 
ATME & \textbf{0.64} & \textbf{0.19} & \textbf{0.14}\\
\midrule
Ground truth & 0.80 & 0.26 & 0.21\\

\bottomrule
\end{tabular}
\end{small}
\end{center}
\caption{FCN scores (higher is better) after training on Cityscapes B$\rightarrow$A at a resolution of $256\times256$. The best result per column is shown in bold.}
\label{tab:fcn_scores}
\end{table}
\subsection{Baselines}
Since our model is built using the pix2pix framework, we train the latter for comparison. Additionally, we train CycleGAN~\cite{zhu2017unpaired} (despite its introduction for unsupervised problems) as a reference of a generator that is trained to have an approximate inverse mapping. The hypothesis is that adding cycle consistency may improve convergence since this restricts the possible paths to equilibrium, with respect to those allowed by the highly under-unconstrained source-to-target mapping originally present in pix2pix. Finally, we train the supervised version of ASGIT~\cite{Lin_2021_WACV} as a reference of a state-of-the-art model using attention in the discriminator, as well as their implementation of UNIT (at a $256\times256$ resolution) using a 2-branch residual attention network~\cite{wang2017residual} in the discriminator.

On the other hand, due to the diffusion in a latent-space representation of the target images being more efficient than in the image space, we choose to train an LDM~\cite{latent_diffusion} conditioned on the source images for comparison.

\subsection{Training details}

We train all GAN models from scratch using the default configuration in pix2pix. That is, we use the Adam optimizer with $\beta_1=0.5$ and $\beta_2=0.999$, with an initial learning rate of $0.0002$ for both the generator (UNet-256) and discriminator (patchGAN) of the vanilla GANs. The learning rate is kept constant in the first 100 epochs and linearly decayed to zero for the following 100 epochs. A batch size of 48 is used and instance normalization. Random jittering and horizontal flipping is applied during training to the images with resolution $256\times256$. For ATME, the UNet-256 is replaced by the UNet in \cref{fig:backbone}(b) with an embedding dimension of $d=64$ and resolutions $R=(1,2,4,8)$.  

On the other hand, we train the LDMs with the default configuration in~\cite{latent_diffusion} for the input resolution $256\times256$. That is, the denoising is done by the UNet in \cref{fig:backbone}(a) after downsampling the input (target) images by a factor of $f=4$ (using the VQ-reg encoder with attention) running the diffusion process, and concatenating the output of this process with a spatially-scaled version of the conditioning (source) image. The diffusion follows a linear schedule of $\beta_t$, from $\beta_1=0.0015$ to $\beta_T=0.0205$ in $T=1000$ timesteps. 




\subsection{Metrics}
We follow recent practices~\cite{MMD2018demystifying,Lin_2021_WACV} and report the Kernel Inception Distance (KID) between feature representations of real and fake images. The feature extraction~\cite{kid_iclr_2018} is done by the Inception v3 model. Additionally, the FCN score~\cite{isola2017image} is computed to further evaluate details of the performance in the Cityscapes dataset. This measures the accuracy of the FCN-8s semantic classifier~\cite{fcn8s} (pre-trained on real images) after segmenting the generated images and comparing the results against the labels these images were synthesized from. 

\begin{figure*}
 \centering
 \includegraphics[width=0.86\linewidth]{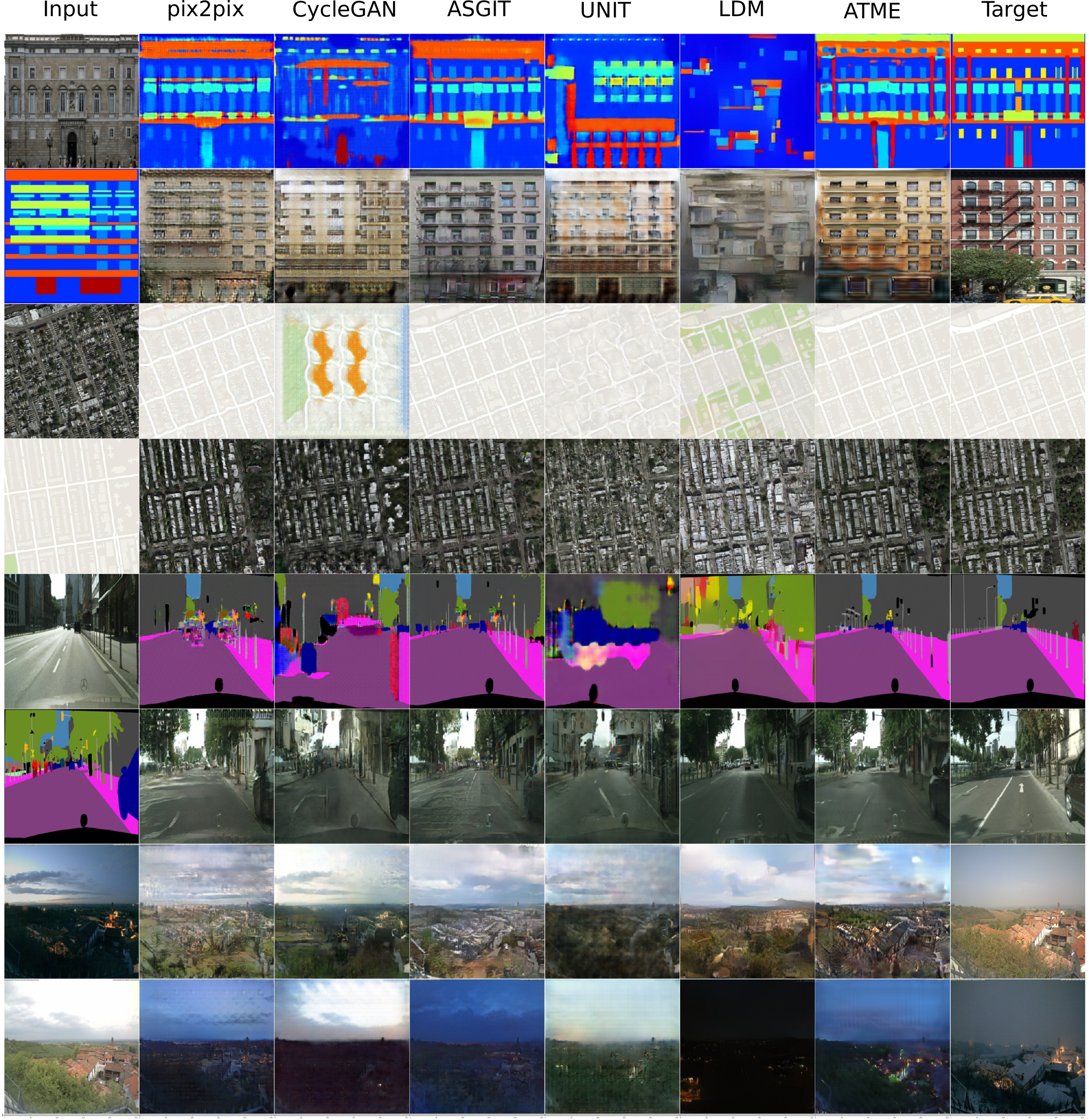}  
 \caption{Sample predictions of all models in the evaluated datasets.}
 \label{fig:mosaic_all}
\end{figure*}

\subsection{Evaluating convergence of GANs}
We keep track of the loss in \cref{eq:patchGAN} at the end of each epoch and notice that, by \cref{th:max_ent} and \cref{eq:gan_obj}, the convergence to equilibrium is manifested as the approach of $-\mathcal{L}_{GAN}(G,D)$ to $\log(4)$ during optimization. This is shown in \cref{fig:losses}, where ATME shows stable convergence in most cases. The cases where convergence seems slower is presumably due to ATME not being large enough --- since GANs are designed to reach Nash equilibrium with a large enough model and infinite data~\cite{goodfellow2016nips} --- or being harder for the generator to remove the input noise in the required number of epochs. The latter is evident from the success in the convergence for the column A$\rightarrow$B in \cref{fig:losses}, which represents the corruption with noise of the photo (much more semantic content than the labels) according to \Cref{tab:dataset_AB}. The exception is the Night2day dataset, for which the photo with more semantic content is in the opposite side B. This slowness in noise removal is further illustrated in \cref{fig:denoising} where, at the same epoch $t=20$, the noise in Cityscapes B$\rightarrow$A still has much more structure than the noise in Maps A$\rightarrow$B. 

\begin{figure*}
 \centering
 \includegraphics[width=0.62\linewidth]{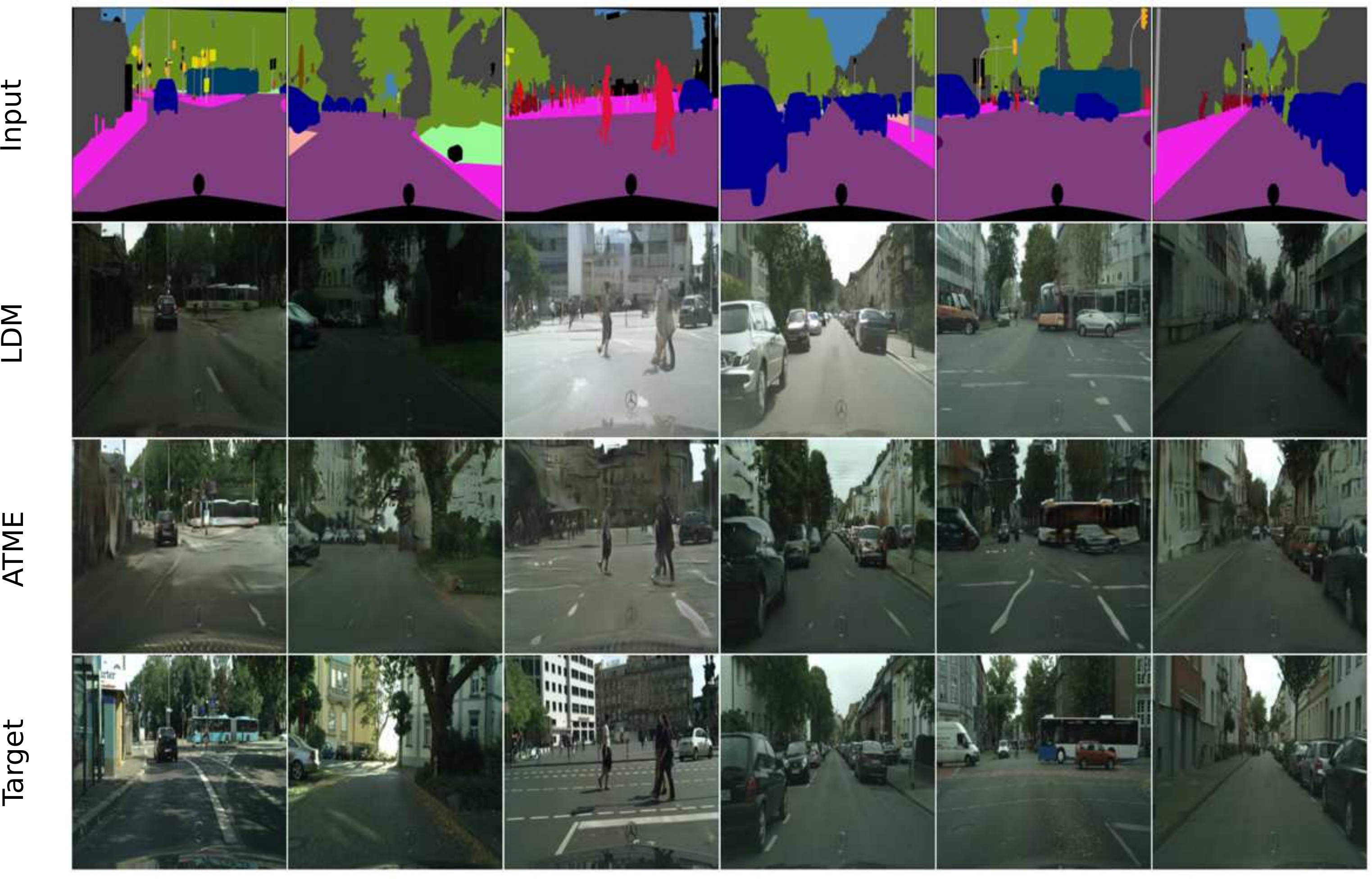}
 \caption{Qualitative predictions of LDM and ATME for a subset of test images in the Cityscapes dataset.}
 \label{fig:mosaic_qualitative}
\end{figure*}

As mentioned above, the model capacity may also play a role, specially for the largest dataset. As seen in \Cref{tab:kid_all}, ATME is the lightest model so it may not be able to accomodate all the variability of the data distribution in this case. We plan to investigate this further in the future. However, preliminary results show that by enlarging ATME to the configuration $d=64$ and $R=(1,1,2,2,4,4,4,8)$, which brings the model to a capacity similar to pix2pix (\ie with $\sim$57M parameters), the worst KID in \Cref{tab:kid_all}, obtained in the task Night2day A$\rightarrow$B, is lowered to 16.3 $\pm$ 2.1, taking ATME from the last place to the top-3 after the bigger CycleGAN and UNIT models. These bigger models were observed to suffer mode collapse for some tasks, as evidenced in \cref{fig:mosaic_all}.

\subsection{Quality of image synthesis}
\Cref{tab:kid_all} shows the KID scores for all models and datasets. Despite being the lightest model, our model ATME shows superior performance than the other methods, assessed as the number of times that it has the lowest KID per task.  

The quality of image generation is further evaluated using the FCN scores in the Cityscapes dataset (see \Cref{tab:fcn_scores}), confirming the superiority of ATME compared to the other methods. Sample predictions from all methods in all datasets are shown in \cref{fig:mosaic_all}.

\subsubsection{Distribution Modes: GANs vs Diffusion models}
Both GANs and diffusion models are trained to learn the target distribution conditioned on the source images. Given an input image $x$, the models are expected to output the most probable image $\hat{y}$ sharing semantic content with $x$. This should have a strong similarity with the ground truth $y$. Although the diffusion models are known to predict images with very high quality, surprisingly for us, the predictions are far from the right mode, as can be seen in \cref{fig:mosaic_qualitative}, \ie LDM not being able to understand the semantics of the right pose (\eg car facing inwards being confused with the car facing outwards), the right contrast, etc. This explains the results of \Cref{tab:fcn_scores} and suggests that GAN models are more appropriate for supervised image-to-image translation than diffusion models.

\section{Conclusion}

We have shown that a significant improvement in the convergence properties of GANs for image-to-image translation is achieved when making the generator and discriminator exhange information symmetrically. We achieve this by informing the generator about the entropic state of the discriminator, as a guide for the equilibration of the adversarial game. The quality of image synthesis is high compared to state-of-the-art methods and our model ATME predicts the modes of the conditional target distribution better than diffusion models.

Several research directions are left open, including exploring a generator model in ATME with higher capacity and, most importantly, extending the method to unsupervised image-to-image translation.



\section*{Appendix}
\setcounter{theorem}{0}
To avoid clutter in notation, we omit the condition on $x$ in the following.
\begin{theorem}
 Let $Y_{i}$ be binary random variables taking on the value $y_{i}=1$ with probability $D_{i}$. If they are statistical independent, the joint distribution $P(Y_1,\cdots,Y_N)$ has maximum entropy if and only if $D_{i}=\tfrac{1}{2}$ for all $i$. In this state, the objective in \cref{eq:patchGAN} reaches the value $-\log(4)$ for an optimal discriminator and generator.
\end{theorem}
\begin{proof}
The joint entropy becomes the sum of the marginal entropies $-\sum_i\sum_{y_i} D_i(y_i)\,\log D_i(y_i)$ if and only if the random variables $Y_1,\cdots,Y_N$ are statistically independent~\cite{Cover.2006}, which is implicit in the patch discriminator being Markovian~\cite{isola2017image}. Now, the entropy of a binary random variable is known to reach a maximum when $D_i(y_i)=\frac{1}{2}$ for all $i$ and $y_i$.
In this case, the objective in \cref{eq:patchGAN} collapses to
 \begin{equation}
  \mathcal{L}_{GAN}=\frac{1}{N}N\left(\log\left( \tfrac{1}{2}\right)+\log\left(\tfrac{1}{2}\right)\right)=-\log(4),
 \end{equation}
 corresponding to the value for an optimal discriminator and generator~\cite{gan_goodfellow_nips_2014}.
\end{proof}


{\small
\bibliographystyle{ieee_fullname}
\bibliography{atme}
}

\end{document}